\documentclass[]{article}

\usepackage{url,tabularx,array,enumerate,amsmath}
\usepackage{graphicx}
\usepackage[english]{babel}
\usepackage[margin=1in]{geometry}
\usepackage{natbib}
\usepackage{amsthm}
\usepackage{enumerate}
\usepackage{algorithm}
\usepackage{algorithmicx}
\usepackage{algpseudocode}
\usepackage{caption}
\usepackage{subcaption}
\usepackage[multiple]{footmisc}

\newtheorem{mydef}{Definition}
\newtheorem{theorem}{Theorem}

\title{Decomposition and Identification of Linear Structural Equation Models} 
\author{Bryant Chen}

\begin{document}

\maketitle

%
%

\begin{abstract}
In this paper, we address the problem of identifying linear structural equation models. We first extend the edge set half-trek criterion to cover a broader class of models. We then show that any semi-Markovian linear model can be recursively decomposed into simpler sub-models, resulting in improved identification power. Finally, we show that, unlike the existing methods developed for linear models, the resulting method subsumes the identification algorithm of non-parametric models.
\end{abstract}

\section{Introduction}

Many researchers, particularly in economics, psychology, and the social sciences, use linear structural equation models (SEMs) to describe the causal and statistical relationships between a set of variables,  predict the effects of interventions and policies, and to estimate parameters of interest.  When modeling using linear SEMs, researchers typically specify the causal structure (i.e. exclusion restrictions and independence restrictions between error terms) from domain knowledge, leaving the structural coefficients (representing the strength of the causal relationships) as free parameters to be estimated from data.  If these coefficients are known, then total effects, direct effects, and counterfactuals can be computed from them directly \citep{pearl:09,chen:pea14}.  However, in some cases, the causal assumptions embedded in the model are not enough to uniquely determine one or more coefficients from the probability distribution, and therefore, cannot be estimated using data.  In such cases, we say that the coefficient is \emph{not identified} or \emph{not identifiable}\footnote{We will also use the term ``identified" with respect to individual variables and the model as a whole.  When the structural equation for a given variable contains one or more unidentified coefficients then we say the variable is unidentified.  Similarly, when a model contains one or more unidentified coefficients than the model is not identified.}.  

Many SEM researchers determine the identifiability of the model by submitting the specification and data to software, which attempts to estimate the coefficients by minimizing a fitting function \citep{kenny2011identification}.  If the model is not identified, then the program will be unable to complete the estimation and warns that the model may not be identified.   While convenient, there are disadvantages to using typical SEM software to determine model identifiability \citep{kenny2011identification,chen:pea14}.  If poor starting values are chosen, the program could mistakenly conclude the model is not identified when in fact it may be identified.  When the model is not identified, the program is not helpful in indicating which parameters are not identified nor are they able to provide estimates for identifiable coefficients\footnote{According to \cite{kenny2011identification}, AMOS is the only program that attempts to estimate parameters when the model is not identified.}.  Most importantly, the program only gives an answer after the researcher has taken the time to collect data.  

Rather than determining the identifiability of parameters by attempting to fit the model, we will detect identifiability directly from the model specification and express identified parameters in terms of the population covariance matrix.  As a result, the modeler can estimate their values, even when the model, as a whole, is not identified.  Additionally, we avoid issues of poor starting values and can determine the identifiability of parameters prior to collecting data.

As compact and transparent representations of the model's structure, causal graphs provide a convenient tool to aid in the identification of coefficients.  First utilized as a causal inference tool by \cite{wright:21}, graphs have more recently been applied to identify causal effects in non-parametric causal models \citep{pearl:2k} and enabled the development of causal effect identification algorithms that are complete for non-parametric models \citep{tian:02,huang:val06,shpitser:pea06c}.  These algorithms can be applied to the identification of coefficients in linear SEMs by identifying non-parametric direct effects, which are closely related to structural coefficients \citep{tian2005identifying,chen:pea14}.  Algorithms designed specifically for the identification of linear SEMs were developed by \cite{brito:pea02a}, \cite{brito:04}, \cite{tian2005identifying,tian:07a,tian2009parameter}, \cite{foygel:12}, and \cite{chen:etal14}.

Surprisingly, none of the aforementioned identification methods subsume the algorithms designed for non-parametric models when naively applied to the identification of linear SEMs\footnote{The coefficient $f$ in Figure \ref{fig:decomp} is identified using non-parametric algorithms but not the aforementioned linear algorithms.}, despite the added assumption of linearity.  In this paper, we first extend the half-trek identification methods of \cite{foygel:12} and \cite{chen:etal14}.  Our extensions can be applied to both semi-Markovian and non-Markovian models.  We then demonstrate how recursive c-component decomposition, which was first utilized in identification algorithms for semi-Markovian non-parametric models \citep{tian:02,huang:val06,shpitser:pea06c}, can be utilized for semi-Markovian linear models.  We show that applying recursive decomposition to our identification method allows the identification of additional coefficients and models.  Further, we will demonstrate that this method subsumes the aforementioned non-parametric identification methods when applied to linear SEMs.  Lastly, we note that our method is not complete and there are identifiable coefficients in certain models that cannot be identified using the methods given.

\section{Preliminaries}


A linear structural equation model consists of a set of equations of the form, 
\begin{equation*}
X = \Lambda X+ \epsilon,
\end{equation*}
where $X = [x_1 , ... , x_n]^t$ is a vector containing the model variables, $\Lambda$ is a matrix containing the \emph{coefficients} of the model, which convey the strength of the causal relationships, and $\epsilon = [\epsilon_1 , ..., \epsilon_n]^{t}$ is a vector of error terms, which represents omitted or latent variables.  The matrix $\Lambda$ contains zeroes on the diagonal, and $\Lambda_{ij} = 0$ whenever $x_i$ is not a cause of $x_j$.  The error terms are random variables and induce the probability distribution over the model variables.  The covariance matrix of $X$ will be denoted by $\Sigma$ and the covariance matrix over the error terms, $\epsilon$, by $\Omega$.

An instantiation of a model $M$ is an assignment of values to the model parameters (i.e. $\Lambda$ and the non-zero elements of $\Omega$).  For a given instantiation $m_i$, let $\Sigma(m_i)$ denote the covariance matrix implied by the model and $\lambda_k (m_i)$ be the value of coefficient $\lambda_j$.  In this paper, we identify the model coefficients from the covariance matrix, $\Sigma$.  
\begin{mydef}
A coefficient, $\lambda_k$, is \emph{identified} if $\lambda_k (m_i) = \lambda_k (m_j)$ whenever $\Sigma(m_i)=\Sigma(m_j)$.
\end{mydef}
In other words, $\lambda_k$ is identified if it can be uniquely determined from the covariance matrix, $\Sigma$.

The causal graph or path diagram of an SEM is a graph, $G=(V,D,B)$, where $V$ are vertices or nodes, $D$ directed edges, and $B$ bidirected edges.  The vertices represent model variables.  Directed eges represent the direction of causality, and for each coefficient $\Lambda_{ij}\neq 0$, an edge is drawn from $x_i$ to $x_j$.  Each directed edge, therefore, is associated with a coefficient in the SEM, which we will often refer to as its structural coefficient.  The error terms, $\epsilon_i$, are not represented in the graph.  However, a bidirected edge between two variables indicates that their corresponding error terms may be statistically dependent while the lack of a bidirected edge indicates that the error terms are independent.  When the causal graph is acyclic without bidirected edges, then we say that the model is \emph{Markovian}.  Graphs with bidirected edges are \emph{non-Markovian}, while acyclic graphs with bidirected edges are additionally called \emph{semi-Markovian}.  

If a directed edge, called $(x, y)$, exists from $x$ to $y$ then we say that $x$ is a parent of $y$.  The set of parents of $y$ is denoted $Pa(y)$.  Additionally, we call $y$ the head of $(x,y)$ and $x$ the tail.  The set of tails for a set of directed edges, $E$, is denoted $Ta(E)$ while the set of heads is denoted $He(E)$.  For a node, $v$, the set of edges for which $He(E)=v$ is denoted $Inc(v)$.  If there exists a path of directed edges from $x$ to $y$ then we say that $x$ is an ancestor of $y$.  The set of ancestors of $y$ is denoted $Anc(y)$.  If $x$ is an ancestor of $y$, then $y$ is a descendant of $x$.  The set of descendants of $x$ is denoted $De(x)$.  Finally, the set of nodes connected to $y$ by a bidirected arc are called the siblings of $y$ or $Sib(y)$. 

A \emph{path} from $x$ to $y$ is a sequence of edges connecting the two vertices.  A path may go either along or against the direction of the edges.  A non-endpoint vertex $w$ on a path is said to be a \emph{collider} if the edges preceding and following $w$ on the path both point to $w$, that is, $\rightarrow w \leftarrow$, $\leftrightarrow w \leftarrow$, $\rightarrow w \leftrightarrow$, or $\leftrightarrow w \leftrightarrow$.  A vertex that is not a collider is a \emph{non-collider}.  

A path between $x$ and $y$ is said to be \emph{unblocked given a set $Z$} (possibly empty), with $x, y\notin Z$ if

\begin{enumerate}
\item every noncollider on the path is not in $Z$ and
\item every collider on the path is in $An(Z)$ \citep{pearl:09}.
\end{enumerate}
%

\section{Extending the Edge Set Half-Trek Criterion}


The half-trek criterion is a graphical condition that can be used to determine the identifiability of recursive and non-recursive linear models \citep{foygel:12}.  \cite{foygel:12} use the half-trek criterion to identify the model variables one at a time, where each identified variable may be able to aid in the identification of other variables.  If any variable is not identifiable using the half-trek criterion, then their algorithm returns that the model is not \emph{HTC-identifiable}.  Otherwise the algorithm returns that the model is identifiable.  Their algorithm subsumes the earlier methods of \cite{brito:pea02a} and \cite{brito:04}.  \cite{chen:etal14} modified the half-trek criterion (calling it the \emph{edge set half-trek criterion}) to identify \emph{connected edge sets}.  Rather than attempting to identify all of a variable's coefficients at once, they instead partition the coefficients according to connected edge sets.  As a result, if one coefficient is not identified, only its connected edge set will not be identified rather than the entire variable.  In this way, they increase the granularity of the criterion to identify additional coefficients in unidentifiable models.  In this section, we will paraphrase some preliminary definitions from \citep{foygel:12} and further extend the edge set half-trek criterion.  Without c-component decomposition, this extension will allow identification of even more coefficients in unidentifiable models.  Coupled with recursive c-component decomposition, our extended half-trek criterion allows the identification of additional models, as well.  

\subsection{Preliminary Definitions}

We begin by establishing a couple preliminary definitions around half-treks.  These definitions and illustrative examples can also be found in \cite{foygel:12} and \cite{chen:etal14}.

\begin{mydef}
\citep{foygel:12} A \emph{half-trek}, $\pi$, from $x$ to $y$ is a path from $x$ to $y$ that either begins with a bidirected arc and then continues with directed edges towards $y$ or is simply a directed path from $x$ to $y$. 
\end{mydef}
\noindent We will denote the set of nodes connected to a node, $v$, via half-treks $htr(v)$.


\begin{mydef}
\citep{foygel:12} For any half-trek, $\pi$, let \emph{Right($\pi$)} be the set of vertices in $\pi$ that have an outgoing directed edge in $\pi$ (as opposed to bidirected edge) union the last vertex in the trek.  In other words, if the trek is a directed path then every vertex in the path is a member of Right($\pi$).  If the trek begins with a bidirected edge then every vertex other than the first vertex is a member of Right($\pi$).
\end{mydef}


\begin{mydef}
\citep{foygel:12} A system of half-treks, ${\pi_1, ..., \pi_n}$, has \emph{no sided intersection} if for all $\pi_i , \pi_j \in \{\pi_1, ..., \pi_n\}$ such that $\pi_i \neq \pi_j$, Right($\pi_i$)$\cap$Right($\pi_j$)$=\emptyset$.
\end{mydef}

\begin{mydef}
\citep{chen:etal14} For an arbitrary variable, $V$, let $Pa_1 , Pa_2 , ... , Pa_k$ be the unique partition of Pa(V) such that any two parents are placed in the same subset, $Pa_i$, whenever they are connected by an unblocked path.  A \emph{connected edge set} with head $V$ is a set of directed edges from $Pa_i$ to $V$ for some $i\in \{1, 2, ..., k\}$.  
\end{mydef}


\subsection{General Half-Trek Criterion}

Having established the preliminary definitions in the previous subsection, we now give our extension of the edge set half-trek criterion.

\begin{mydef}
(General Half-Trek Criterion) Let $E$ be a set of directed edges sharing a single head $v$.  A set of variables $Y$ satisfies the \emph{general half-trek criterion} with respect to $E$, if 

\begin{enumerate}[(i)]
\item $|Y|=|E|$,
\item $Y\cap({v}\cup Sib(v))=\emptyset$, 
\item There is a system of half-treks with no sided intersection from $Y$ to $Ta(E)$, and
\item $(Pa(v)\setminus Ta(E))\cap htr(Y) = \emptyset$.
\end{enumerate}
\end{mydef}

A set of directed edges, $E$, sharing a head $v$ is identifiable if there exists a set, $Y_E$, that satisfies the general half-trek criterion (g-HTC) with respect to $E$, and $Y_E$ consists only of ``allowed" nodes.  Intuitively, a node, $y$, is not allowed if \begin{enumerate}[(i)]
\item it is half-trek reachable from $He(E)$ and the coefficients of $y$ that lie on the half-treks from $He(E)$ are not identifiable or
\item $y$ is connected to $Pa(v)\setminus Ta(E)$, and the coefficients of $y$ that lie on unblocked paths between $Pa(v)\setminus Ta(E)$ are not identifiable.
\end{enumerate}
Otherwise, $y$ is allowed.  The following definition formalizes this notion.

\begin{mydef}
\label{def:allow}
Let $E$ be the set of directed edges in a causal graph, $G$.  We say that a node, $z$, is \emph{g-HT-allowed} (or simply \emph{allowed}) for directed edges $E_v\subseteq E$ with head $v$ if 
\begin{enumerate}[(i)]
\item
\begin{enumerate} [(a)]
\item $z$ is not half-trek reachable from $v$ and 
\item $z$ is not connected to $Pa(v)\setminus Ta(E)$
\end{enumerate} 
OR 
\item there exists a partition of $E$, $ES$, such that $E_v \in ES$,  an ordering on $ES$, $\prec$, and a family of subsets $(Y_{E_i})$, one subset for each $E_i \prec E_v$, such that $Y_{E_i}$ satisfies the g-HTC with respect to $E_i$ and $E_j\prec E_k$ for $(E_j, E_k \prec E_v)$ whenever
\begin{enumerate}[(a)]
\item $He(E_j)\subseteq htr(He(E_k))\cap Y_{E_k}$ or
\item $Y_{E_k}$ is connected to $Pa(He(E_k))\setminus Ta(E_k)$ via edges in the set $Inc(Y_{E_k}) \cap E_j$
\end{enumerate}
and any directed edges belonging to $z$ that lie on a half-trek from $v$ to $z$ or lie on a path between $z$ and $Pa(He(E_k))\setminus Ta(E_k)$ belong to a set $E_z \subset ES$ ordered before $E_v$.
\end{enumerate}
\end{mydef}

If $Y$ is a set of allowed variables for $E_v$ that satisfies the half-trek criterion with respect to $E_v$, we will say that $Y$ is an \emph{g-HT-admissible} set for $E_v$.  We are now ready to use the g-HTC to identify coefficients.

\begin{theorem}
\label{thm:htID}
If a g-HT-admissible set for directed edges $E_v$ with head $v$ exists then $E_v$ is identifiable.  Further, let $Y_{E_v}=\{y_1 , ..., y_k\}$ be a g-HT-admissible set for $E_v$, $Ta(E_v)=\{p_1 , ..., p_k\}$, and $\Sigma$ be the covariance matrix of the model variables.  Define $\mathbf{A}$ as 
\begin{align}
\mathbf{A_{ij}}=\begin{cases}
[(I-\Lambda)^T \Sigma]_{y_i p_j}, & y_i \in htr(v) \mathrm{\;or\;}  y_i \mathrm{\;connected}\\
& \mathrm{to\;}Pa(v)\setminus Ta(E_v) , \\
\Sigma_{y_i p_j}, & y_i \notin htr(v)
\end{cases}
\end{align}
and  $\mathbf{b}$ as
\begin{align}
\mathbf{b_{i}}=\begin{cases}
[(I-\Lambda)^T \Sigma]_{y_i v}, & y_i \in htr(v) \mathrm{\;or\;} y_i \mathrm{\;connected}\\
& \mathrm{to\;} Pa(v)\setminus Ta(E_v), \\
\Sigma_{y_i v}, & y_i \notin htr(v)
\end{cases}
\end{align}
Then $\mathbf{A}$ is an invertible matrix and $\mathbf{A} \cdot \Lambda_{Ta(E_v),V}= \mathbf{b}$.  
\end{theorem}
\begin{proof}
The proof for this theorem is similar to the proof of Theorem 1 in \cite{foygel:12}.  Rather than giving a complete proof, we simply explain why our changes are valid. The g-HTC identifies arbitrary sets of directed edges belonging to a node rather than all of the directed edges belonging to a node.  It is able to do this because of two changes.  First, sets that contain nodes that are connected to $Pa(v) \setminus Ta(E)$ via half-treks cannot be half-trek admissible for $E$ (see Definition 6).  As a result, the paths from half-trek admissible set, $Y_E$, to $v$ travel only through coefficients of $E$ and no other coefficients of $E$.  This ensures that $\mathbf{A}\cdot \Lambda_{Ta(E),v} =\mathbf{b}$.  Second, nodes that are connected to $Pa(v)\setminus Ta(E)$ are not allowed unless their coefficients that lie on paths to $Pa(v)\setminus Ta(E)$ are identified.  Likewise, nodes that are half-trek reachable from $v$ are not allowed unless their coefficients that lie on the half-treks from $v$ are identified.  This ensures that $\mathbf{A}$ and $\mathbf{b}$ are computable.  Other coefficients need not be identified because they will vanish from $\mathbf{A}$ and $\mathbf{b}$ during the computations, $((I-\Lambda)^T \cdot \Sigma)_{y_i p_j}$ and $((I-\Lambda)^T \cdot \Sigma)_{y_i v}$, due to zeroes in the matrix $\Sigma$.
\end{proof}


The g-HTC impoves upon the edge set HTC because subsets of connected edge sets may be identifiable even when the connected edge set as a whole is not.  

\begin{theorem}
\label{thm:g-HTC}
Any coefficient identifiable using the edge set HTC is also identifiable using the g-HTC.  Moreover, the g-HTC is able to identify coefficients in certain models that the edge set HTC is not.
\end{theorem}
\begin{proof}
First, we note that when the set of edges in question is a connected edge set, then the edge set HTC and the g-HTC are equivalent.  As a result, any set of edges identifiable using the edge set HTC is identifiable using the g-HTC, and we have proven the first part of the Theorem.  

We will now use Figures \ref{fig:ID} and \ref{fig:ID2} to show that certain coefficients that are not identifiable using the edge set HTC are identifiable using the g-HTC.  The edge set HTC identifies coefficients associated with connected edge sets altogether.  As a result, the coefficient $b$ in the graph depicted in Figure \ref{fig:ID} is not identifiable using the edge set HTC since it belongs to the connected edge set $\{V_2\rightarrow V_3 , V_4\rightarrow V_3\}$ and the coefficient $d$ is not identifiable.  However, $b$ is identifiable using the g-HTC since $V_1$ is g-HT-admissible for $b$. 

Additionally, coefficient $a$ in Figure \ref{fig:ID2} is not identifiable using the edge set half-trek criterion due to the fact that $c$ is not identifiable
.  Using the g-HTC and Theorem \ref{thm:htID}, we can first identify $b$ since $V_2$ is a g-HT-admissible set for $b$.  Once $b$ is identified, we can use $Y_a = V_1$ to identify $a$.
\end{proof}

\begin{figure}
\centering
\begin{subfigure}[t]{0.23\textwidth}
\includegraphics[width=\textwidth]{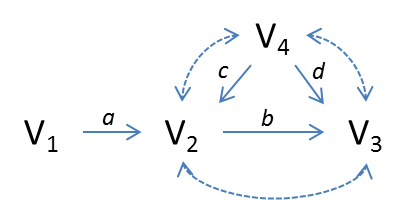}
\caption{}
\label{fig:ID}
\end{subfigure}
\begin{subfigure}[t]{0.17\textwidth}
\includegraphics[width=\textwidth]{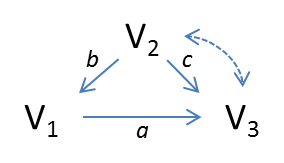}
\caption{}
\label{fig:ID2}
\end{subfigure}
\caption{(a) $b$ is identified using the g-HTC but not the edge set HTC (b) $a$ is identified using the g-HTC but not the edge set HTC}
\end{figure}

\subsection{g-HTC Algorithm}

Finding a g-HT-admissible set for directed edges, $E$, with head, $v$, from a set of allowed nodes, $A_E$, can be accomplished by utilizing the max-flow algorithm described in \cite{chen:etal14}\footnote{\cite{brito:04} utilized a similar max-flow construction in his identification algorithm.  This was later modified by \cite{foygel:12} for the half-trek criterion and finally \cite{chen:etal14} for the edge set half-trek criterion. }, which we call $\mathrm{MaxFlow}(G_f (E,A))$.  This algorithm returns a maximal set of allowed nodes that satisfies (ii) - (iv) of the g-HTC. 
%
%

In some cases, there may be no g-HT-admissible set for $E^{'}$ but there may be one for $E\subset E^{'}$.  In other cases, there may be no g-HT admissible set of variables for a set of edges $E$ but there may be a g-HT admissible set of variables for $E^{'}$ with $E\subset E^{'}$.  As a result, if a HT-admissible set does not exist for $E_v$, where $E_v = Inc(v)$ for some node $v$, we may have to check whether such a set exists for all possible subsets of $E_v$ in order to identify as many coefficients in $E_v$ as possible.  This process can be somewhat simplified by noting that if $E$ is a connected edge set with no g-HT-admissible set, then there is no superset $E^{'}$ with a g-HT-admissible set.  

Algorithm \ref{alg:ID} utilizes the g-HTC and Theorem \ref{thm:htID} to identify as many coefficients in a recursive or non-recursive linear SEM as possible.  It iterates through every unidentified connected edge set in the graph, attempting to identify each unidentified subset of its edges using Theorem \ref{thm:htID}.  For a given set of directed edges, $E$, the initial set of HT-allowed nodes is the set of nodes that is not half-trek reachable from $He(E)$.  As more and more edges are identified, the set of allowed nodes increases so that a set of edges that was not initially identified may become identifiable as the algorithm iterates.  We define $\mathrm{Allowed}(E, \mathrm{IDEdges},G)$ to be the set of nodes that are either (i) not half-trek reachable from $He(E)$ and not connected to $Pa(He(E))\setminus Ta(E)$ or (ii) their edges that lie on half-treks from $He(E)$ or paths to $Pa(He(E))\setminus Ta(E)$ are in $\mathrm{IDEdges}$.  As a result, Algorithm \ref{alg:ID} attempts to identify each connected edge set (and its subsets) until all coefficients have been identified or no new coefficients are identified in a given iteration.  

\begin{algorithm}[H]
\caption{HT-ID($G,\Sigma, \mathrm{IDEdges}$)}
\label{alg:ID}
\begin{algorithmic}
\State {\bfseries Initialize:} $\mathrm{EdgeSets}\leftarrow$ all connected edge sets in $G$
\Repeat
	\ForAll{$ES$ in $\mathrm{EdgeSets}$}
		\State$v\leftarrow He(ES)$
			\ForAll{$E\subset ES$ such that $E\not\subset \mathrm{IDEdges}$}
				\State $A_E \leftarrow \mathrm{Allowed}(E, \mathrm{IDEdges}, G)$
				\State $Y_E \leftarrow \mathrm{MaxFlow}(G_f (v,A_E))$
				\If{$|Y_E|=|Ta(E)|$}
					\State Identify $E$ using Theorem \ref{thm:htID}
					\State $\mathrm{IDEdges} \leftarrow \mathrm{IDEdges} \cup E$
				\EndIf
			\EndFor
	\EndFor
\Until{All coefficients have been identified or no coefficients have been identified in the last iteration}\\
\Return IDEdges
\end{algorithmic}
\end{algorithm}

\section{Decomposition}
\label{sec:c-comp}

\begin{figure}
\centering
\begin{subfigure}[b]{.35\textwidth}
\includegraphics[width=\textwidth]{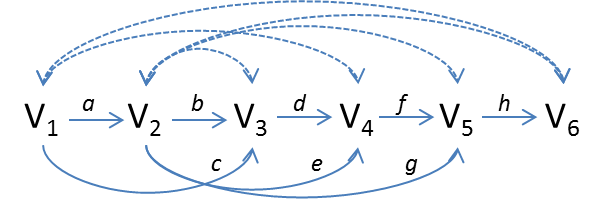}
\caption{}
\label{fig:decomp}
\end{subfigure}
\begin{subfigure}[b]{.3\textwidth}
\includegraphics[width=\textwidth]{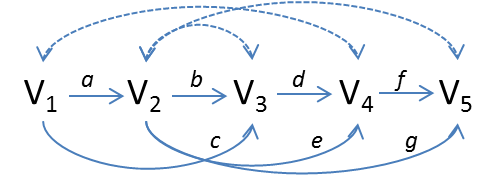}
\caption{}
\label{fig:decomp_marg}
\end{subfigure}

\begin{subfigure}[b]{.3\textwidth}
\includegraphics[width=\textwidth]{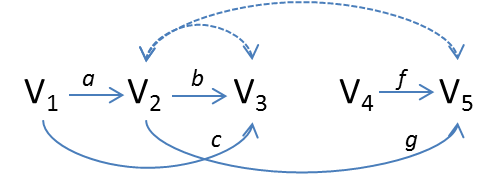}
\caption{}
\label{fig:decomp1}
\end{subfigure}
\begin{subfigure}[b]{.235\textwidth}
\includegraphics[width=\textwidth]{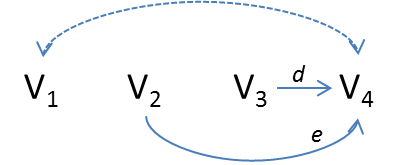}
\caption{}
\label{fig:decomp2}
\end{subfigure}
\caption{(a) The graph is not identified using the g-HTC and cannot be decomposed (b) After removing $V_6$ we are able to decompose the graph (c) Graph for c-component, $\{V_2 , V_3 , V_5\}$ (d) Graph for c-component, $\{V_1 , V_4\}$}
\end{figure}
%

\cite{tian2005identifying} showed that the identification problem could be simplified in semi-Markovian linear structural equation models by decomposing the model into sub-models according to their c-components, defined below.  Each coefficient is identifiable if and only if it is identifiable in the sub-model to which it belongs.  \cite{foygel:12} noted that first decomposing the model in this way and then applying the half-trek criterion allows for the identification of a larger set of models.  In this section, we will briefly review the results of \cite{tian2005identifying} before showing that the c-component decomposition can be applied recursively to the model after marginalizing certain variables.  This idea was first used to identify interventional distributions in non-parameteric models by \cite{tian:02} and \cite{tian:pea02} and adapting this technique for linear models will allow us to identify models that the g-HTC, even coupled with (non-recursive) c-component decomposition, is unable to identify.  Further, it ensures the identification of all coefficients identifiable using methods developed for non-parametric models--a guarantee that none of the existing methods developed for linear models satisfy.  

\subsection{C-Component Decomposition}

\begin{mydef}
\citep{tian2005identifying} The \emph{c-component} of a node $v$ in a causal graph, denoted $C(v)$, is the set of nodes that is connected to $v$ by paths consisting only of bidirected arcs.
\end{mydef}

\cite{tian2005identifying} showed that a coefficient is identified if and only if it is identified in the sub-graph consisting of its c-component and the parents of the c-component.  More formally, a coefficient from $x$ to $y$ is identified if and only if it is identified in the sub-model constructed in the following way:
\begin{enumerate}[(i)]
\item The sub-model variables consist of $C(y)\cup Pa(C(y))$.
\item The structural equations for the variables in $C(y)$ are the same as their structural equations in the original model while the structural equations for the parents simply equate each parent with its error term.  
\item If the error terms of any two variables in the sub-model were uncorrelated in the original model then they are uncorrelated in the sub-model.
\end{enumerate}

When we decompose the model according to c-components, the distribution over the variables in the sub-model, $W = S\cup Pa(S)$, for a c-component $S$, is not simply the marginal distribution over $W$.  Instead, the distribution, which we denote $P_S$, is computed from the joint distribution over the model variables, $P(V)$.  Letting $V=\{v_1 , v_2 , ... , v_n\}$ be ordered topologically, we have 
\begin{equation}
\label{eq:P}
P_S (W) = \prod_{\{i|v_i \in S\}} P(v_i | V^{(i-1)}) \prod_{\{l | v_l \in Pa(S)\}} P(v_l),
\end{equation}
where $V^{(i-1)} = \{v_1 , v_2 , ... , v_{i-1}\}$.

For example, consider the graph depicted in Figure \ref{fig:decomp_marg}.  We have two c-components, $S_1 = \{v_2 , v_3 , v_5\}$ and $S_2 = \{v_1 , v_4\}$, with \begin{align*}P_{S_1} = &P(v_2 | v_1) P(v_3 | v_2 , v_1) P(v_5 | v_4 , v_3 , v_2 , v_1) \\ &P(v_1) P(v_4)\end{align*} and \begin{equation*}P_{S_2} =  P(v_1) P(v_4 | v_3 , v_2 , v_1)P(v_3)P(v_2)P(v_3).\end{equation*} 

We will denote the sub-model for a c-component, $S$, $M_S$ and the graph it induces $G_S$.  $G_{S_1}$ is depicted in Figure \ref{fig:decomp1} and $G_{S_2}$ in Figure \ref{fig:decomp2}.

\begin{theorem}
\citep{tian2005identifying} Let a variable $v_j$ be in a c-component $S_j$ in a SEM $M$.  A coefficient, $\Lambda_{ij}$, is identifiable if and only if it is identifiable in the model $M_{S_j}$.  
\end{theorem}


\subsection{Recursive C-Component Decomposition}

The graph in Figure \ref{fig:decomp} consists of a single c-component, and we are unable to decompose it.  As a result, we are able to identify $a$ but no other coefficients using the extended half-trek criterion.  Moreover, $f = \frac{\partial}{\partial v_4}E[v_5|do(v_6, v_4 , v_3 , v_2 , v_1)]$ is identified using identification methods developed for non-parametric models (e.g. do-calculus) but not the g-HTC or other methods developed for linear models.  

However, if we remove $v_6$ from the analysis, then the resulting model can be decomposed.  Let $M$ be the model depicted in Figure \ref{fig:decomp}, $P(v)$ be the distribution induced by $M$, and $M^{'}$ be a model that is identical to $M$ except the equation for $v_6$ is removed.  $M^{'}$ induces the distribution $\int_{v_6} P(V) d v_6$, and its associated graph $G^{'}$ yields two c-components, as shown in Figure \ref{fig:decomp_marg}.  

Now, decomposing $G^{'}$ according to these c-components yields the sub-models depicted by Figures \ref{fig:decomp1} and \ref{fig:decomp2}.  Both of these sub-models are identifiable using the half-trek criterion.  Thus, all coefficients other than $h$ have been shown to be identifiable.  Returning to the graph prior to removal, depicted in Figure \ref{fig:decomp}, we are now able to identify $h$ because both $v_4$ and $v_5$ are now allowed nodes for $h$, and the model is identified\footnote{While $v_4$ and $v_5$ are technically not allowed according to Definition \ref{def:allow}, they can be used in g-HT-admissible sets to identify $h$ using Theorem \ref{thm:htID} since their coefficients have been identified.}.

As a result, we can improve our identification algorithm by recursively decomposing, using the extended half-trek criterion, and removing descendant sets\footnote{Only removing descendant sets have the ability to break up -components.  For example, removing $\{v_2\}$ from Figure \ref{fig:decomp} does not break the c-component because removing $v_2$ would relegate its influence to the error term of its child, $v_3$.  As a result, the graph of the resulting model would include a bidirected arc between $v_3$ and $v_6$, and we would still have a single c-component.}.  Note, however, that we must consider every descendant set for removal.  It is possible that removing $D_1$ will allow identification of a coefficient but removing a superset $D_2$ with $D_1 \subset D_2$ will not.  Additionally, it is possible that removing $D_2$ will allow identification but removing a subset $D_1$ will not.  

After recursively decomposing the graph, if some of the removed variables were unidentified, we may be able to identify them by returning to the original graph prior to removal since we may have a larger set of allowed nodes.  For example, we were able to identify $h$ in Figure \ref{fig:decomp} by ``un-removing" $v_6$ after the other coefficients were identified.  In some cases, however, we may need to again recursively decompose and remove descendant sets.  As a result, in order to fully exploit the powers of decomposition and the g-HTC, we must repeat the recursive decomposition process on the original model until all marginalized nodes are identified or no new coefficients are identified in an iteration.

Algorithm \ref{alg:recID} decomposes the graph according to its c-components and then applies HT-ID (Algorithm \ref{alg:ID}) to each sub-model.  If there are still unidentified coefficients, then it removes descendant sets and decomposes again.  The whole process is repeated until every coefficient is identified or no new coefficients are identified in an iteration.  $\Sigma_{P_{S_i}}$ is the covariance matrix of $P_{S_i}$, where $S_i$ is a c-component.  $\Sigma_{V\setminus D_i}$ is the covariance matrix after marginalizing $D_i$ from $\Sigma$.  Finally, $G_{V\setminus D_i}$ is the graph with the set $D_i$ removed.  


\begin{algorithm}[H]
\caption{Decomp-HT-ID($G, \Sigma$)}
\label{alg:recID}
\begin{algorithmic}
\State {\bfseries Initialize:} $\mathrm{IDEdges}\leftarrow \emptyset$
\Repeat
	\State $\mathrm{IDEdges} \leftarrow \mathrm{IDEdges}\cup$Rec-Decomp$(G, \Sigma, \mathrm{IDEdges}$)
\Until{All coefficients have been identified or no coefficients have been identified in the last iteration}
\State
\Return IDEdges

\State
\Procedure{Rec-Decomp}{$G, \Sigma, \mathrm{IDEdges}$)}
	\State $V\leftarrow$ vertices in $G$
	\State $\mathrm{Edges}\leftarrow$ all edges in $G$
	\ForAll {c-component, $S_i$, in $G$}
		\State $\mathrm{IDEdges} = \mathrm{IDEdges} \cup$ HT-ID$(G_{S_i}, \Sigma_{S_i}, \mathrm{IDEdges})$
	\EndFor
	\If{$\mathrm{IDEdges}=\mathrm{Edges}$}
		\State Return $\mathrm{IDEdges}$
	\Else
		\ForAll {descendant set, $D_i$, in $G$}
			\State $\mathrm{IDEdges}\leftarrow \mathrm{IDEdges}\cup $Rec-Decomp$(G_{V\setminus D_i} , \Sigma_{V\setminus D_i}, \mathrm{IDEdges}$
		\EndFor
	\EndIf
\State
\Return IDEdges
\EndProcedure
\end{algorithmic}
\end{algorithm}



We now show that any direct effect identifiable using non-parametric methods is also identified using Algorithm \ref{alg:recID}. 

\begin{theorem}
\label{thm:npid}
Let $M$ be a linear SEM with variables $V$.  Let $M^{'}$ be a non-parametric SEM with identical structure to $M$.  If the direct effect of $x$ on $x$ for $x,y \in V$ is identified in $M^{'}$ then the coefficient $\Lambda_{xy}$ in $M$ is identified using Algorithm \ref{alg:recID}.
\end{theorem}

\begin{proof}
Let $G$ be the causal graph of $M$ and $M^{'}$.  Suppose the direct effect of $x$ on $y$ is identified in $M^{'}$.  Then according to Theorem 3 of \citep{shpitser:08}, there does not exist a subgraph of $G$ that is a $y$-rooted c-tree \citep{shpitser:08}.  This implies that $MACS(y) = y$.  By recursively decomposing the graph into c-components and marginalizing descendant sets, we can obtain a graph where only $MACS(y)$ and its parents remain in the graph.  Since $MACS(Y)=y$, the parents of $y$ in this graph represent a g-HT admissible set that allows the identification of all coefficients of $y$.  
\end{proof}

\section{Conclusion}

In this paper, we extend the edge set half-trek criterion \citep{chen:etal14}. We then incorporate recursive c-component decomposition \citep{tian:02}, and show that the resulting identification method is able to identify more coefficients and models than the existing algorithms developed for both linear and non-parametric models.

\section{Acknowledgments}

I would like to thank Jin Tian and Judea Pearl for helpful comments and discussions.  This research was partly supported by NSF \#IIS-1302448. 

\bibliography{book}  
\bibliographystyle{ims}
\end{document}